\documentclass[mlmain]{jmlr}

\usepackage{mathtools}
\usepackage{booktabs}
\usepackage{multirow}
\usepackage{microtype}

\DeclareMathOperator{\tr}{tr}

\jmlrvolume{}
\firstpageno{1}
\editors{}

\jmlryear{2026}
\jmlrworkshop{Symmetry and Geometry in Neural Representations}
\jmlrproceedings{PMLR}{Proceedings of Machine Learning Research}

\title[GaugeQuant]{GaugeQuant: Online Learning of Quantization-Optimal Bases from LLM Symmetries}

\author{%
  \Name[Bento]{Miguel P. Bento\thanks{Accepted at the Symmetry and Geometry in Neural Representations Workshop (NeurReps), NeurIPS 2026. This work was conducted independently of the authors' roles at Feedzai.}} \Email{miguel.pedra@feedzai.com}\\
  \addr Feedzai, Lisbon, Portugal
  \AND
  \Name[Seabra]{João F. Seabra}\footnotemark[1] \Email{joao.seabra@feedzai.com}\\
  \addr Feedzai, Lisbon, Portugal
}

\begin{document}

\maketitle

\begin{abstract}
  Transformers are known to have internal continuous symmetries that leave outputs invariant while altering how activations are distributed across coordinates, and hence how well they quantize. GaugeQuant leverages this in-training by introducing a LogSumExp term to the loss that breaks the symmetries, thus selecting a basis that minimizes activation outliers. A stop-gradient operator ensures that only rotation matrices are updated, yielding the language modeling objective completely unaltered. Our quantization strategy requires no specific calibration data, no quantization simulation, and adds a small inference overhead. As a proof of concept, we evaluate on WikiText-2 perplexity and zero-shot reasoning. On Qwen-2.5 0.5B, GaugeQuant reduces perplexity (e.g.\ $25.62 \to 16.76$ under W4A16) and improves zero-shot accuracy. On LLaMA-2 7B, perplexity improves but downstream accuracy does not. The code is available at: \url{https://github.com/MPedraBento/gauge-quant}
\end{abstract}


\section{Introduction}
\label{sec:intro}

The capabilities of Large Language Models (LLMs) are only attained nowadays with massive computational and memory resources. In fact, the transformer~\citep{vaswani2017attention} architectures at the core of state of the art LLMs contain billions of parameters, requiring multiple GPUs for both training and inference~\citep{zhao2026surveylargelanguagemodels}. 

To reduce the memory footprint of LLMs, several quantization techniques have been put forward where different variables (e.g. weights and activations) of the model are compressed into lower-precision formats. Post-training quantization (PTQ) methods are the most commonly seen in the literature, where model variables are quantized after the LLM is fully trained. The scale factors needed to pack the data into fewer bits are computed using small calibration datasets, making this kind of quantization procedures computationally inexpensive. However, PTQ also leads to degradation in the LLM performance, caused by the presence of outliers in the quantized variables distribution. Especially usual among activations~\citep{dettmers2022llmint8, 10.5555/3618408.3619993}, outliers stretch the quantization range, leaving fewer bits available for representing the most frequent values. 

In this paper, we propose GaugeQuant: an in-training quantization mechanism that mitigates PTQ's drawbacks by learning a preferred basis across different boundaries of the LLM architecture during training.\footnote{We use ``gauge'' in the sense of a redundant internal parametrization that leaves observables invariant, analogous to gauge freedom in field theory, though here the symmetry is global (position-independent) at each boundary rather than local in the differential-geometric sense.} The main source of inspiration for our approach is SpinQuant~\citep{Liu2024SpinQuantLQ}, where outliers are suppressed by learning the best orthogonal rotations of the network variables for that effect. Those rotations can be performed because of LLMs' symmetries, which make their outputs invariant under such transformations. While retaining this idea, we introduce a different loss function that removes outliers during training and the need for static calibration datasets. Besides the usual cross-entropy loss, which minimizes the difference between the model's predicted probability distribution and the true probability distribution of the next token, we include a second term that penalizes outliers among the rotated boundary activations. More specifically, the latter takes the form of a LogSumExp function: a smooth, dense approximation of the infinity norm~\citep{10.1007/s10107-004-0552-5} that routes gradient to every dimension, and not just the single maximum outlier. 

This paper is structured as follows: In the following section, we introduce ideas from previous works that helped us shape our approach to in-training quantization. In Section~\ref{sec:symmetries}, we do an overview of the LLM symmetries. The most important aspects of the LLM training procedure using GaugeQuant are presented in Section~\ref{sec:intraining}. Our results, obtained on Qwen-2.5 0.5B~\citep{qwen2025qwen25technicalreport} and LLaMA-2 7B~\citep{touvron2023llama} are shown in Section~\ref{sec:results}. A discussion of our results is made in Section~\ref{sec:discussion}. In Section~\ref{sec:limitations}, we describe the limitations of our approach. Finally, in Section~\ref{sec:conclusions}, we draw some concluding remarks about our work.                                                                                                                    

\section{Background and related work}
\label{sec:background}

\textbf{LLM PTQ}. In the context of LLMs, quantization is mostly motivated by the prohibitive cost of deploying those models for inference. Earlier approaches like ZeroQuant~\citep{NEURIPS2022_adf7fa39}, LLM.int8()~\citep{dettmers2022llmint8} and nuQmm~\citep{Park2022nuQmmQM} tackled very large models at reasonable runtimes by rounding just the weights to the nearest quantization level. GPTQ~\citep{Frantar2022GPTQAP}, which is commonly used as a PTQ baseline in more recent works, brought a considerable improvement over those methods, allowing the usage of a single GPU for performing inference with Generative Pre-trained Transformer (GPT) models with 175 billion parameters.

Unlike GPTQ where only the model weights are quantized, more recent PTQ schemes provide mechanisms to quantize activations as well. SmoothQuant~\citep{10.5555/3618408.3619993} and AWQ~\citep{lin2023awq} are some of the earliest examples of those schemes, in which the quantization difficulty is migrated between activations and weights.

\textbf{Exploiting LLM symmetries:} Neural network symmetries are a research topic of their own that gives valuable insights about parameter redundancies within models (see for example~\citep{Hashimoto:2024rms}). Those redundancies have been exploited by several authors to reduce the complexity of LLMs. For example, SliceGPT~\citep{Ashkboos2024SliceGPTCL} introduces a post-training sparsification scheme, which takes advantage of the LLM computational invariance -- the designation given by the authors to the output's invariance under any orthogonal transformation of the weight matrices. In the context of quantization, symmetries are usually exploited for reducing outliers. Taking inspiration from the computational invariance described in SliceGPT, Quarot~\citep{Ashkboos2024QuaRotO4} applies randomized Hadamard rotations at internal boundaries of the model, in order to eliminate activation outliers. RoLoRA~\citep{Huang2024RoLoRAFR} couples this idea with Low-Rank Adaptation (LoRA)~\citep{Hu2021LoRALA} to enhance the fine-tuning efficiency of the model. SpinQuant~\citep{Liu2024SpinQuantLQ} takes a step forward in outlier suppression by letting the model learn the rotations that optimize the performance of the (quantized) LLM. To this end, SpinQuant employs the Cayley Stochastic Gradient Descent (SGD), which is especially proficient at optimizing orthonormal matrices~\citep{li2020efficientriemannianoptimizationstiefel}. Building upon Quarot and SpinQuant, ReSpinQuant~\citep{Kim2026ReSpinQuantEL} assigns a rotation matrix for each layer. The advent of reasoning LLMs also motivated the proposal of ParoQuant~\citep{liang2026paroquant}: a PTQ scheme tailored to these models, which combines channel-wise scaling with rotations.

While all the works mentioned thus far apply quantization post-training, some authors have put forward quantization-aware training (QAT) techniques that also remove outliers through rotations of model parameters. One example is RoSTE~\citep{wei2025roste}, which combines these rotations with quantization-aware supervised fine-tuning. Another example is HALO~\citep{halo2025}, where Hadamard rotations are placed in both the forward and backward passes. Finally, it is worth mentioning~\citep{Park2025OutlierSafePF}, which proposes an outlier-safe in-training framework that, like GaugeQuant, targets quantization robustness during training, though via optimizer and normalization changes rather than learned rotations.

\section{Symmetries of LLMs}
\label{sec:symmetries}

Multi-head attention (MHA) can be described for a single head $h$. Given input $X \in \mathbb{R}^{N \times d_{\text{model}}}$, the network computes
\begin{equation}
    Q^{(h)} = X W_Q^{(h)}, \quad K^{(h)} = X W_K^{(h)}, \quad V^{(h)} = X W_V^{(h)} \, ,
\end{equation}
where $W_Q^{(h)}, W_K^{(h)}, W_V^{(h)} \in \mathbb{R}^{d_\text{model} \times d_k}$ and $d_k$ is the head dimension. With Rotary Positional Embeddings (RoPE)~\citep{su2024roformer}, the pre-softmax scores are
\begin{equation}
    S_{m,n}^{(h)} = \frac{(X_m W_Q^{(h)} \Theta_m) (X_n W_K^{(h)} \Theta_n)^\top}{\sqrt{d_k}} \, ,
\end{equation}
where $\Theta_m$ is the RoPE matrix (block-diagonal 2D rotations) at position $m$. The head output is
\begin{equation}
    Y^{(h)} = \text{softmax}(S^{(h)}) \, V^{(h)} \, W_O^{(h)} \, ,
\end{equation}
with $W_O^{(h)} \in \mathbb{R}^{d_k \times d_\text{model}}$. The SwiGLU MLP block~\citep{Shazeer2020GLUVI} computes
\begin{equation}
    \text{MLP}(X) = \underbrace{\left[ \sigma(X W_{\text{gate}}) \odot (X W_{\text{up}}) \right]}_{H} W_{\text{down}} \, ,
\end{equation}
where $\sigma$ is SiLU and $\odot$ the Hadamard product.

\subsection{Invariances of the transformer}

These linear maps possess continuous symmetries, transformations that leave the network output exactly unchanged.

\paragraph{V--O boundary.} Let $R_V \in SO(d_k)$. Attention computes a weighted sum of value vectors, so the rotation passes through linearly
\begin{equation}
    \text{softmax}(S^{(h)})(V^{(h)} R_V)(R_V^\top W_O^{(h)}) = Y^{(h)} \, .
\end{equation}
This reveals a continuous $SO(d_k)$ symmetry per KV head. For Grouped Query Attention (GQA), where $n_{\text{heads}}/n_{\text{kv}}$ query heads share each KV head, all heads in a group share the same rotation, yielding $SO(d_k)^{n_{\text{kv}}}$ total degrees of freedom.

\paragraph{Q--K boundary.} A rotation $R_Q$ must commute with RoPE: $R_Q \Theta_m = \Theta_m R_Q$. Since $\Theta_m$ is composed of 2D rotations, $R_Q$ is restricted to a block-diagonal matrix of independent planar rotations, reducing the symmetry from $SO(d_k)$ to the torus $U(1)^{d_k/2}$. These are too few parameters to meaningfully suppress outliers, so we do not exploit this smaller symmetry.

\paragraph{MLP boundary.} The linear map $H W_{\text{down}}$ admits the same invariance
\begin{equation}
    H W_{\text{down}} = (H R)(R^\top W_{\text{down}}) \, , \quad R \in SO(d_{\text{ff}}) \, .
\end{equation}
However, unlike the attention case, $R$ cannot be fully absorbed into weights. Absorbing $R$ upstream would require $\sigma(X W_{\text{gate}} R) \odot (X W_{\text{up}} R) = [\sigma(X W_{\text{gate}}) \odot (X W_{\text{up}})] R$, which fails because: (i) SiLU is element-wise but $R$ mixes dimensions, so $\sigma(v R) \neq \sigma(v) R$; and (ii) the Hadamard product satisfies $(a R) \odot (b R) \neq (a \odot b) R$ for non-trivial rotations. Therefore $R^\top$ can be permanently fused into $W_{\text{down}}$, but the multiplication $H \cdot R$ must remain as a runtime operation, analogously to~\citep{Liu2024SpinQuantLQ, Ashkboos2024QuaRotO4}.

\subsection{Basis choice for quantization}

These symmetries mean the parameter space contains degrees of freedom with no effect on exact outputs, they represent a choice of coordinate basis. Let $\Theta$ denote standard parameters and $\Theta_R$ the transformed parameters. At full precision
\begin{equation}
    \hat{y}(x; \Theta_R) = \hat{y}(x; \Theta) \, .
\end{equation}
However, uniform quantization $Q(\cdot)$ is element-wise and non-linear, so it does not commute with rotations
\begin{equation}
    \hat{y}(x; Q(\Theta_R)) \neq \hat{y}(x; Q(\Theta)) \, .
\end{equation}
The quantization grid step $\Delta \propto \|v\|_\infty$ is dictated by the worst outlier (see Appendix~\ref{app:quantization}). The optimal basis is therefore the one that distributes activation mass uniformly across dimensions. The central question is how to find this basis efficiently during training.

\section{In-training gauge optimization}
\label{sec:intraining}

Post-training methods such as SpinQuant~\citep{Liu2024SpinQuantLQ} optimize $R$ by minimizing the task loss of the quantized network on a static calibration set, with the pretrained weights held frozen, where the loss is the cross-entropy of the network with the quantizer $Q(\cdot)$ inserted into the forward pass. Because $Q(\cdot)$ is non-differentiable and sits inside the computation of the loss, this requires straight-through gradient estimation and a frozen model with calibration data, thus fundamentally incompatible with in-training optimization, where the weights are themselves moving.

\subsection*{A differentiable proxy for quantization}

We replace the explicit quantization simulation with a smooth proxy. Since $\Delta \propto \|v\|_\infty$, suppressing the $L_\infty$ norm of the rotated activations directly reduces quantization error. We use the LogSumExp as a dense, differentiable approximation
\begin{equation}\label{eq:l_rot}
    \mathcal{L}_{\text{rot}}(R) = \frac{1}{\beta} \log \left( \sum_{i} \exp\left(\beta \, |z_i| \right) \right) \, , \quad z = \text{sg}(H) \cdot R \, ,
\end{equation}
where $H$ denotes the boundary activations ($H = XW_V^{(h)}$ for attention, $H = \sigma(XW_{\text{gate}}) \odot (XW_{\text{up}})$ for MLP) and $\text{sg}(\cdot)$ is the stop-gradient operator and $\beta > 0$ controls the approximation. The gradient
\begin{equation}
    \frac{\partial \mathcal{L}_{\text{rot}}}{\partial |z_i|} = \frac{\exp(\beta\, |z_i|)}{\sum_j \exp(\beta\, |z_j|)} \, ,
\end{equation}
smoothly addresses outliers while yielding non-zero values for smaller numbers.\footnote{We note that the gradient is exactly the softmax function with inverse temperature $\beta$.}
Eq.~\eqref{eq:l_rot} is not rotationally invariant, and thus, adding it to the loss \emph{explicitly breaks} the symmetry, selecting a quantization-optimal basis.

The stop-gradient operator ensures a clean separation between the language modeling and gauge objectives
\begin{equation}
    \mathcal{L}_{\text{total}} = \mathcal{L}_{\text{CE}}(X; \Theta) + \lambda \, \mathcal{L}_{\text{rot}} \, ,
\end{equation}
which guarantees
\begin{equation}
    \nabla_W \mathcal{L}_{\text{total}} = \nabla_W \mathcal{L}_{\text{CE}}, \quad \nabla_R \mathcal{L}_{\text{total}} = \lambda \, \nabla_R \mathcal{L}_{\text{rot}} \, .
\end{equation}
The weights are updated by the cross-entropy loss alone, while the rotation parameters are updated exclusively by LogSumExp. The rotation acts only on a stop-gradient copy of the activations and is not inserted into the weights' forward path during training; it is fused in only at evaluation. The two objectives are orthogonal by construction.

\subsection*{A symmetry of the training}

Up to this point, we have discussed the symmetries as special orthogonal matrices and purposely ignored the role of weight decay. So, why must the gauge be a special orthogonal matrix rather than an arbitrary invertible matrix? The answer lies in weight decay. Standard training regularizes with $\|W\|_F^2$, which is invariant under orthogonal transformations, but \emph{not} under general linear maps. Therefore, $O(d)$ is the maximal continuous symmetry group that is simultaneously an invariance of the forward pass, norm-preserving for activations and an invariant of weight decay (see Appendix~\ref{app:weight_decay}). 

While full orthogonal matrices could be used, naive gradient descent on $R$ leaves the orthogonal manifold after a single step. Instead, we learn a skew-symmetric matrix $A = -A^\top$ and map it to $SO(d)$ via the Cayley transform
\begin{equation}
    R = (I - A)(I + A)^{-1} \, .
\end{equation}
Thus, without loss of generality we choose $SO(d) \subset O(d)$, which admits a smooth Cayley parametrization.

\paragraph{Block-diagonal structure for the MLP.} A full $SO(d_{\text{ff}})$ rotation requires $O(d_{\text{ff}}^3)$ matrix inversion, prohibitive for $d_{\text{ff}} = 11{,}008$. We parametrize the MLP gauge as a block-diagonal rotation
\begin{equation}
    R_{\text{MLP}} = \bigoplus_{k=1}^{d_{\text{ff}}/b} R_k \, , \quad R_k \in SO(b) \, ,
\end{equation}
with block size $b = 64$. Each block is independently mapped via its own Cayley transform with cost $O(b^3)$, giving a total cost of applying it in inference of $O(d_{\text{ff}} \cdot b)$, linear in the hidden dimension. The block structure is itself orthogonal, preserving the gauge invariance. All matrix inversions and weight fusions are computed in \texttt{float32} to prevent orthogonality drift from reduced-precision arithmetic.

\section{Empirical studies}
\label{sec:results}

We evaluate GaugeQuant on two models: Qwen-2.5 0.5B~\citep{qwen2025qwen25technicalreport} and LLaMA-2 7B~\citep{touvron2023llama}. Training streams text from C4~\citep{DBLP:journals/corr/abs-1910-10683} for 8192 steps with batch size 1, sequence length 512, and $\lambda = 0.1$. This corresponds to a short continued-training run (${\sim}4$M tokens); integration into full-scale pretraining is left for future work. The rotation loss learning rate is $2 \times 10^{-4}$, while base weights use $2 \times 10^{-5}$ (full fine-tuning for Qwen, LoRA $r{=}16$ for LLaMA-2). The evolution of the $\mathcal{L}_{\text{rot}}$ during training is shown in Figure~\ref{fig:convergence}. Evaluation measures perplexity on WikiText-2~\citep{Merity2016PointerSM} and zero-shot accuracy on five downstream reasoning tasks: ARC-Easy~\citep{allenai:arc}, ARC-Challenge~\citep{allenai:arc}, HellaSwag~\citep{zellers2019hellaswag}, PIQA~\citep{Bisk2020} and WinoGrande~\citep{ai2:winogrande}. As a lightweight feasibility check at low compute, perplexity is measured on a single 2048-token window of the test split (evaluation length 2048 vs.\ training length 512) rather than the full test set; a benchmarked evaluation is left to future work.

\begin{figure}[htbp]
\centering
\includegraphics[width=0.60\textwidth]{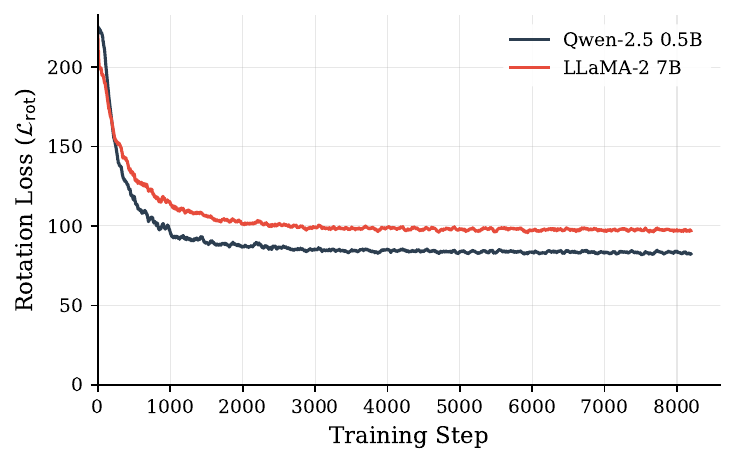}
\caption{Rotation loss ($\mathcal{L}_{\text{rot}}$) during training. The penalty decreases from ${\sim}200$ to ${\sim}100$ over 8192 steps, showing that the learned rotation progressively suppresses activation outliers.}
\label{fig:convergence}
\end{figure}

\paragraph{Quantization setup.} We simulate symmetric min-max quantization on all seven linear layers (\texttt{q/k/v/o\_proj}, \texttt{gate/up/down\_proj}). Weights are quantized per output channel, activations per token. We report three regimes: W4A16 (4-bit weights, full-precision activations), W4A4 g128 (4-bit activations with group size 128) and W4A4 per-token (4-bit activations with a single scale per token, the strictest setting). We also report BF16 (no quantization) as an invariance check.

\paragraph{Baselines.} For each regime we report three conditions: (1)~\textbf{Baseline}: the pretrained HuggingFace model quantized directly; (2)~\textbf{Control}: fine-tuned with the same training budget but without the gauge loss; (3)~\textbf{GaugeQuant}: fine-tuned jointly with gauge optimization.

\paragraph{Overhead.} The attention V--O gauge fuses entirely into $W_V$ and $W_O$, leaving no runtime cost. Only the MLP rotation stays online as one block-diagonal multiply $H \cdot R$, $b=64$, at cost $O(d_{\text{ff}} b)$. Measured at LLaMA-2 7B dimensions, it adds $3.2\text{--}5.1\%$ to the full MLP block (GPU, \texttt{fp16}/\texttt{bf16}, $2048$--$8192$ tokens), timed in isolation. SpinQuant and QuaRot pay a similar online rotation here but via a fused Hadamard transform, cheaper than our learned rotation, so we claim no latency parity. At training time the gauge adds a per-step cost dominated by the exact Cayley transform. This is a one-time cost of a short continued-training run and requires no calibration stage. The per-step cost is not fundamental as replacing the explicit inverse with the fixed-point Cayley iteration used in SpinQuant's Cayley SGD removes it, which we leave as an engineering optimization.

\subsection{Results}

\begin{table}[htbp]
\centering
\caption{WikiText-2 perplexity under symmetric fake quantization. GaugeQuant can reduce quantization induced perplexity. The \textbf{BF16} column reports full-precision perplexity as a control, as the rotations are near-lossless in full precision.}
\label{tab:main_results}
\vspace{0.5em}
\begin{tabular}{llcccc}
\toprule
\textbf{Model} & \textbf{Method} & \textbf{BF16} & \textbf{W4A16} & \textbf{W4A4 g128} & \textbf{W4A4 per-tok} \\
\midrule
\multirow{3}{*}{Qwen-2.5 0.5B}
  & Baseline       & 9.24  & 25.62  & 187.8 & 16{,}462 \\
  & Control (FT)   & 9.56  & 27.69  & 199.3 & 12{,}637 \\
  & \textbf{GaugeQuant}  & 9.56  & \textbf{16.76} & \textbf{61.2} & \textbf{1{,}430} \\
\midrule
\multirow{3}{*}{LLaMA-2 7B}
  & Baseline       & 4.07  & 11.16  & 8.22 & 2{,}397 \\
  & Control (FT)   & 4.08  & 13.22  & 9.25 & 5{,}099 \\
  & \textbf{GaugeQuant}  & 4.08  & \textbf{5.45}  & \textbf{6.73} & \textbf{1{,}923} \\
\bottomrule
\end{tabular}
\end{table}

\paragraph{W4A16.} GaugeQuant improves W4A16 perplexity on both models (LLaMA-2 7B $11.16 \to 5.45$; Qwen-2.5 0.5B $25.62 \to 16.76$). Although $\mathcal{L}_{\text{rot}}$ targets activation outliers directly, the fused rotation also transforms the weight matrices ($W_V^{\text{new}} = R^\top W_V$, $W_{\text{down}}^{\text{new}} = W_{\text{down}} R$), redistributing weight values and reducing weight quantization error as a secondary effect.

\paragraph{W4A4 g128.} Under group-128 activation quantization, GaugeQuant reduces perplexity from 8.22 to 6.73 on LLaMA-2 7B (18\% reduction) and from 187.8 to 61.2 on Qwen-2.5 0.5B (67\% reduction). Fine-tuning alone (Control) provides no benefit and can worsen quantization performance, confirming that the improvement is attributable to the gauge rotation. Some orderings are unexpected (e.g.\ the LLaMA-2 Baseline being lower at W4A4 g128 than at W4A16), which we do not investigate here.

\paragraph{W4A4 per-token.} Per-token activation quantization (one scale across the full hidden dimension) remains a regime where all methods fail: both models collapse to unusable perplexity regardless of rotation. The relative improvements confirm that the gauge suppresses activation outliers, but viable W4A4 inference requires group quantization as shown above.

\paragraph{BF16.} In the absence of quantization, perplexity is always the same for Control and GaugeQuant. Moreover, the differences observed between the baseline and those two conditions are marginal in both models. On Qwen-2.5 0.5B, such perplexity difference is of the order $10^{-1}$, while for LlaMA-2 7B it is one order of magnitude smaller.

\begin{table}[htbp]
\centering
\small
\caption{Zero-shot accuracy under W4A4 g128 quantization. GaugeQuant improves average accuracy on Qwen-2.5 0.5B. On LLaMA-2 7B it does not improve over the Baseline or Control.}
\label{tab:downstream}
\vspace{0.5em}
\begin{tabular}{llcccccc}
\toprule
\textbf{Model} & \textbf{Method} & \textbf{ARC-E} & \textbf{ARC-C} & \textbf{HellaS.} & \textbf{PIQA} & \textbf{WinoG.} & \textbf{Avg.} \\
\midrule
\multirow{3}{*}{Qwen-2.5 0.5B}
  & Baseline              & 34.89 & 23.98 & 32.06 & 56.86 & 50.91 & 39.74 \\
  & Control (FT)          & 36.32 & \textbf{24.32} & 31.63 & 56.20 & 51.30 & 39.95 \\
  & \textbf{GaugeQuant}   & \textbf{39.14} & 23.55 & \textbf{35.22} & \textbf{59.36} & \textbf{53.99} & \textbf{42.25} \\
\midrule
\multirow{3}{*}{LLaMA-2 7B}
  & Baseline              & \textbf{63.17} & 38.48 & 66.99 & 72.91 & 60.38 & 60.39 \\
  & Control (FT)          & 61.49 & \textbf{38.57} & \textbf{67.34} & \textbf{73.83} & \textbf{63.06} & \textbf{60.86} \\
  & \textbf{GaugeQuant}   & 60.10 & 36.95 & 66.41 & 72.63 & 60.93 & 59.40 \\
\bottomrule
\end{tabular}
\end{table}

\paragraph{Downstream reasoning tasks.} We test both Qwen-2.5 and Llama-2 on five distinct downstream tasks, summarized on Table~\ref{tab:downstream}. On Qwen-2.5 0.5B model, GaugeQuant improves average accuracy by 2.3 points over the control, with improvements across four of five benchmarks. On LLaMA-2 7B, accuracy is not improved by our rotation strategy. We note that the LLaMA-2 experiment uses LoRA ($r{=}16$) while Qwen uses full fine-tuning.

\section{Discussion and Broader Implications}
\label{sec:discussion}

We show that the continuous symmetries of transformer architectures can be used \emph{online} during training, rather than only post-hoc on frozen models. We stress that GaugeQuant is not positioned as a drop-in replacement for post-training rotation methods such as SpinQuant, but as a \emph{complementary} stage that operates inside the training loop, with no calibration data and no quantization simulation. Because both act on the same $SO(d)$ gauge orbit, the two compose naturally, as we can integrate GaugeQuant into any training run to reach a quantization-friendly basis, then apply SpinQuant post-hoc to refine it against a specific calibration set. We note that a potential advantage of learning the basis in-training is exposure to the full training stream rather than a small static calibration set. We do not measure calibration sensitivity here, and a controlled comparison against different calibrations is left to future work. A full empirical head-to-head and composition study requires calibration pipelines beyond the scope of this work. More broadly, this framework generalizes beyond quantization as any hardware constraint that breaks rotational symmetry could in principle be targeted by a similar proxy loss.

The smaller improvement for Llama-2 in comparison with Qwen-2.5 also deserves a few remarks. The absence of a proxy to suppress weight outliers, together with a larger number of parameters in Llama-2, might help explain the lack of improvement in that model because statistically, it is more likely to have a bigger number of outliers in it.

\section{Limitations}
\label{sec:limitations}

Our method relies on an online MLP rotation that adds a small cost to inference. This arises because the SwiGLU non-linearity prevents absorbing $R$ upstream, requiring a block-diagonal rotation ($b=64$) at runtime. This block structure is itself a design choice that limits outlier redistribution across block boundaries. 

Another limitation is that the LogSumExp proxy suppresses all large activations indiscriminately, which may reduce the model's capability to encode distinctions between similar outputs. An additional proxy targeting weight outliers could mitigate this effect.

Furthermore, perplexity is reported on a single 2048-token window rather than the full test set, a choice due to limited compute. While our results are promising on perplexity, downstream tasks point to milder results on Llama-2 7B. We do not isolate the factors for this gap, such as QK rotations, residual stream rotations, LoRA itself, choice of $\beta$, among many other possible improvements.       

Finally, experiments are limited to two models (0.5B and 7B), with larger models possibly exhibiting different outlier dynamics.

\section{Conclusion}
\label{sec:conclusions}

We introduced GaugeQuant, a method that exploits the continuous symmetries of transformer architectures to learn quantization-optimal bases during training. By adding a differentiable LogSumExp loss term that explicitly breaks the gauge symmetry, the method selects a basis that suppresses activation outliers without interfering with the language modeling objective. Under our protocol, GaugeQuant reduces quantization-induced perplexity on both models and improves zero-shot accuracy on Qwen-2.5 0.5B, while downstream accuracy on LLaMA-2 7B is not improved. We do not yet isolate the cause of this difference. Future work includes a full benchmarked evaluation in more than one test set, combining GaugeQuant with weight-only outlier suppression, and scaling to larger models.

\bibliography{biblio}


\appendix

\section{Quantization error and the $L_\infty$ norm}
\label{app:quantization}

We briefly justify why minimizing $\|v\|_\infty$ reduces quantization error. Consider symmetric uniform quantization of a vector $v \in \mathbb{R}^d$ to $b$ bits. The quantized representation is
\begin{equation}
    \hat{v}_i = \Delta \cdot \text{clamp}\!\left(\left\lfloor \frac{v_i}{\Delta} \right\rceil,\, -2^{b-1},\, 2^{b-1}-1\right), \quad \Delta = \frac{\|v\|_\infty}{2^{b-1} - 1} \, .
\end{equation}
The per-element quantization error is bounded by $|v_i - \hat{v}_i| \leq \Delta/2$. The total squared error satisfies
\begin{equation}
    \|v - \hat{v}\|_2^2 \leq d \cdot \frac{\Delta^2}{4} = \frac{d}{4(2^{b-1}-1)^2} \|v\|_\infty^2 \, .
\end{equation}
Therefore, for fixed $d$ and $b$, the worst-case quantization error is directly proportional to $\|v\|_\infty^2$. An orthogonal rotation $R$ preserves $\|v\|_2$ but can reduce $\|v R\|_\infty$ (redistributing mass from outlier dimensions). For any single vector $v$, $\|vR\|_\infty \geq \|v\|_2/\sqrt{d}$ for every orthogonal $R$, with equality when all coordinates of $vR$ have equal magnitude. Since the gauge $R$ is shared across all tokens, this per-vector value is a lower bound that is generally not attainable for all tokens simultaneously.

Note that this bound is tight only when $v$ has no clipped entries. In practice, for highly non-uniform $v$, the actual quantization error depends on the full distribution of $v_i$, not just $\|v\|_\infty$. Furthermore, this analysis addresses activation quantization. Weight quantization error depends on $\|W R\|_\infty$ along the channel dimension, which our current proxy does not directly penalize (see Section~\ref{sec:limitations}).

\section{Weight decay breaks $GL(d)$ into $O(d)$ as the maximal symmetry}
\label{app:weight_decay}

We prove that $O(d)$ is the largest continuous matrix group whose action on weights is a symmetry of the standard weight decay regularizer.

\begin{proposition}
Let $G \subseteq GL(d)$ be a connected Lie subgroup such that $\|WM\|_F = \|W\|_F$ for all $W \in \mathbb{R}^{m \times d}$ and all $M \in G$. Then $G \subseteq O(d)$.
\end{proposition}

\begin{proof}
The Frobenius norm satisfies $\|WM\|_F^2 = \tr(M^\top W^\top W M)$. For this to equal $\|W\|_F^2 = \tr(W^\top W)$ for all $W$, we require
\begin{equation}
    \tr(M^\top S M) = \tr(S) \quad \forall\, S \succeq 0 \, ,
\end{equation}
where $S = W^\top W$ ranges over all positive semidefinite matrices as $W$ varies. By the cyclic property of the trace, this is equivalent to $\tr(S \cdot MM^\top) = \tr(S)$, or
\begin{equation}
    \tr\bigl(S(MM^\top - I)\bigr) = 0 \quad \forall\, S \succeq 0 \, .
\end{equation}
Suppose $MM^\top - I \neq 0$. Then $MM^\top - I$ is a nonzero symmetric matrix and therefore has at least one nonzero eigenvalue $\lambda \neq 0$ with corresponding eigenvector $v$. Choosing $S = vv^\top \succeq 0$, we obtain $\tr(vv^\top(MM^\top - I)) = v^\top(MM^\top - I)v = \lambda \neq 0$, which is a contradiction. Therefore $MM^\top = I$, hence $M \in O(d)$.
\end{proof}

Since the Cayley transform parametrizes $SO(d)$, our restriction to this subgroup is without loss of generality. The LogSumExp penalty depends on $|z_i|$, which is invariant under sign flips, so every minimum within $O(d)$ is attainable within $SO(d)$.

\end{document}